\documentclass{article}

\usepackage{arxiv}
\usepackage[square,sort,comma,numbers]{natbib}

\usepackage[utf8]{inputenc} 
\usepackage[T1]{fontenc}    
\usepackage{hyperref}       
\usepackage{url}            
\usepackage{booktabs}       
\usepackage{amsfonts}       
\usepackage{nicefrac}       
\usepackage{microtype}      
\usepackage{xcolor}         
\usepackage{enumitem}
\usepackage{wrapfig}

\usepackage{graphicx} 
\usepackage{bm}
\usepackage{amsmath} 
\usepackage{algorithm}
\usepackage{algorithmic}
\usepackage{amsthm}
\usepackage{mathrsfs}
\usepackage{comment}
\usepackage{caption}
\usepackage{amssymb}

\newcounter{ass-counter}
\newcounter{thm-counter}
\newcounter{remark-counter}
\newtheorem{theorem}[thm-counter]{Theorem}
\newtheorem{lemma}[thm-counter]{Lemma}

\newtheorem{assumption}[ass-counter]{Assumption}

\def\ouralg{\texttt{PASTO}}

\def\E{\mathbb E}
\def\bp{\mathbf p}

\def\g{g}
\def\bg{\bm{g}}
\def\bp{\bm{p}}
\def\xi{\bm{xi}}
\def\E{\mathbb E}
\def\hmu{\hat{\bm{\mu}}}
\def\bmu{\bm{\mu}}

\def\bmu{\bm{\mu}}
\def\hg{{g}}
\newcommand\numberthis{\addtocounter{equation}{1}\tag{\theequation}}

\usepackage{xcolor}

\def\E{\mathbb E}
\def\bp{\mathbf p}
\usepackage{authblk}

\author[1]{{Weicong Ding}\thanks{weicongding@kuaishou.com}~}
\author[2]{{Hanlin Tang}\thanks{tanghl1994@gmail.com}}
\author[3]{{JingShuo Feng}\thanks{jingsf@uw.edu}}
\author[1]{Lei Yuan}
\author[1]{Sen Yang}
\author[1]{Guangxu Yang}
\author[1]{Jie Zheng}
\author[1]{Jing Wang}
\author[1]{Qiang Su}
\author[1]{Dong Zheng}
\author[1]{Xuezhong Qiu}
\author[1]{Yongqi Liu}
\author[1]{Yuxuan Chen}
\author[1]{Yang Liu}
\author[1]{Chao Song}
\author[1]{Dongying Kong}
\author[1]{Kai Ren}
\author[1]{Peng Jiang}
\author[1]{Qiao Lian}
\author[1]{{Ji Liu}\thanks{ji.liu.uwisc@gmail.com}}
\affil[1]{Kuaishou Technology}
\affil[2]{University of Rochester}
\affil[3]{University of Washington}

\title{\ouralg: Strategic Parameter Optimization in Recommendation Systems -- Probabilistic is Better than Deterministic}

\begin{document}

\maketitle

\begin{abstract}
Real-world recommendation systems often consist of two phases. In the first phase, multiple predictive models produce the probability of different immediate user actions. In the second phase, these predictions are aggregated according to a set of `\emph{strategic parameters}'  to meet a diverse set of business goals, such as longer user engagement, higher revenue potential, or more community/network interactions. In addition to building accurate predictive models, it is also crucial to optimize this set of `\emph{strategic parameters}' so that primary goals are optimized while secondary guardrails are not hurt. In this setting with multiple and constrained goals, this paper discovers that \emph{a probabilistic strategic parameter regime can achieve better value compared to the standard regime of finding a single deterministic parameter}. The new probabilistic regime is to learn the best distribution over strategic parameter choices and sample one strategic parameter from the distribution when each user visits the platform. 
To pursue the optimal probabilistic solution, we formulate the problem into a stochastic compositional optimization problem, in which the unbiased stochastic gradient is unavailable.
Our approach is applied in a popular social network platform with hundreds of millions of daily users and achieves $+0.22\%$ lift of user engagement in a recommendation task and  $+1.7\%$ lift in revenue in an advertising optimization scenario comparing to using the best deterministic parameter strategy.
\end{abstract}
 
\section{Introduction}
\label{sec:intro}
%
Consider the pipeline to generate a list of personalized contents (e.g., videos, ads) to users in a real (probably oversimplified) recommendation system.  Such pipelines typically include two phases: 1) predicting the probability of user's immediate actions and 2) calculating a rank score for final recommendation, as illustrated in Figure~\ref{fig:example}.
\begin{figure}[!hbt] 
\centerline{\includegraphics[width=0.8\textwidth]{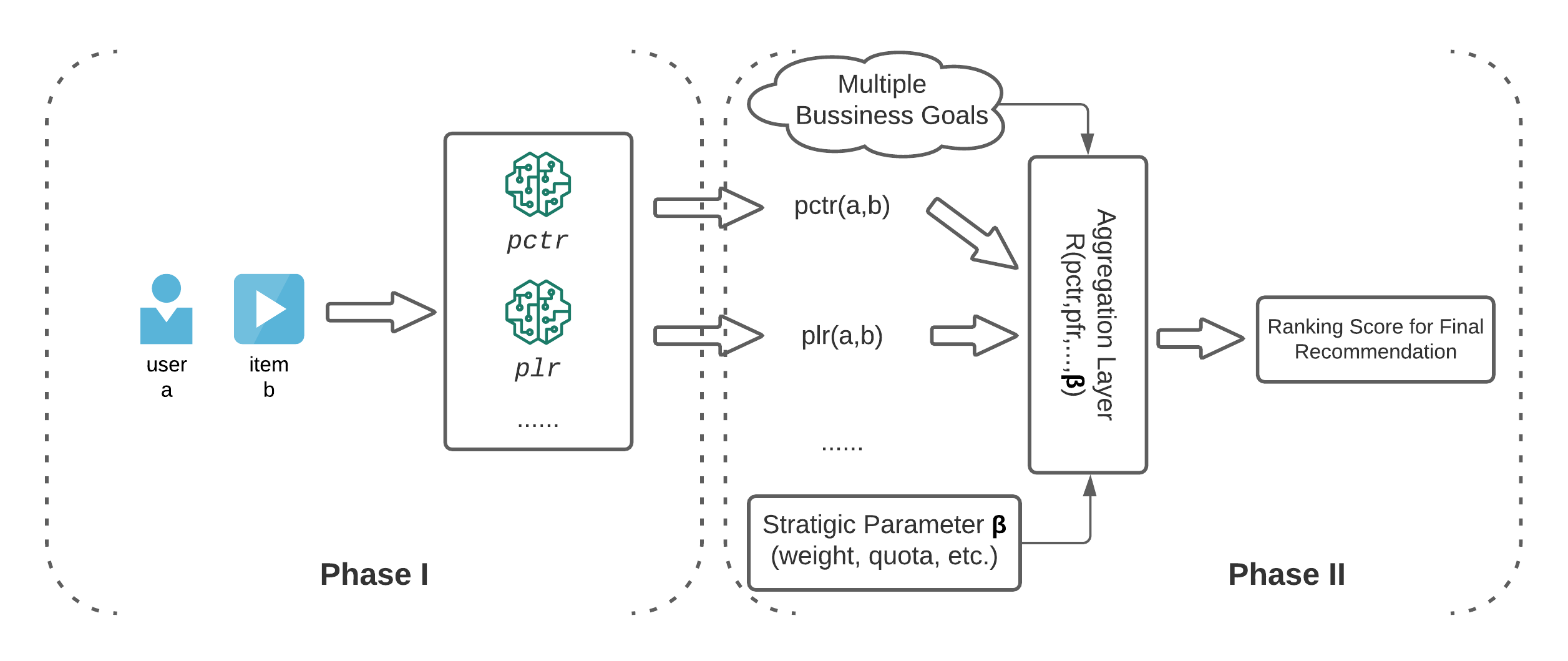}}
\caption{\small A illustrating example of strategic parameters in a simplified recommendation system with two phases. The 1st phase has multiple predictors for click-through-rate (pctr), like-rate (plr), etc. The 2nd phase is an aggregation layer with strategic parameter $\bm\beta$ to generate the final recommendation score. Note that the strategic parameters can be in many other components of a recommendation system. 
}
\label{fig:example} 
\end{figure}
In the first phase, the goal is to accurately predict various immediate user actions such as clicking on an ad or liking a recommended content. The prediction functions to these actions can be directly learned from the user's behavior data. However, it is not optimal to use any single one of these predicted scores to rank and display the final contents because the ultimate goal of recommender is to improve \emph{multiple} downstream business goals. 
For example, the goals could be to maximize the average user engagement time while the average content like rate does not decrease. These business purposes are usually the global and overall metric, which can hardly be modeled by users' immediate actions and are not directly learnable. Therefore, in the second phase, one needs an aggregation function to compute an overall score for each user-item pair $(a,b)$ based on the predictions from the first phase, 
\[
    R_{\bm\beta}(a, b) = R(\text{pctr}(a,b), \text{plr}(a,b), \cdots; \bm\beta),
\]
where $\text{pctr}(a,b)$, $\text{plr}(a,b)$ are the predicted content click through rate, like rate, etc. ${\bm \beta}$ is the hyper parameter of the function $R()$, which is referred as the \textbf{strategic parameter} in recommendation. 
A simple version of $R_{\bm\beta}(\cdot)$ can be a weighted sum \citep{ma2018entire,zhao2019recommending,tang2020progressive} where ${\bm\beta}$ are the mixing weights. Our paper focuses on the strategic parameters. 

Examples like Figure~\ref{fig:example} is ubiquitous in industrial recommendation systems. The strategic parameter $\bm\beta$ can be in many components of a practical recommendation system and are often in the form like weights, thresholds, or quotas. 
While the problem of building predictive models (e.g, models in the 1st phase in Figure~\ref{fig:example}) has been extensively studied\citep{gomez2015netflix,  zhang2019deep},  deciding the strategic parameter has not yet received equal attention. It is mainly because the problem is generally classified into a \emph{standard} black-box optimization problem -- tuning the strategic parameters to optimize the overall business metrics. Manual tuning via controlled experiments are popular in practice \citep{zhao2019recommending, rubinstein2013cross,suzuki2020multi} and more sophisticated methods like bandit optimization \citep{bubeck2012regret, graepel2010web}, evolution algorithms \citep{chen2013combinatorial}, Pareto-efficiency \citep{pei2019value,lin2019pareto,milojkovic2019multi}, and  Bayesian optimization \citep{letham2019constrained,kandasamy2018parallelised, belakaria2019max} have been explored. All of these existing methods pursue a \emph{deterministic} (or single) strategic parameter; that is, the same strategic parameter $\bm\beta$ is applied to all the requests/users to the recommendation system. 

In this paper, we argue that the deterministic (single) strategic parameter is \emph{not} the optimal solution. We discover that a \emph{probabilistic} choice for the strategic parameter can be superior to any deterministic one, especially when there are multiple business goals to pursue simultaneously. In the probabilistic solution, we learn an {optimal distribution} over all the candidate strategic parameter choices. The system then works as follows: when there is a request to the recommender (e.g., a user visit), we first sample one of the multiple strategic parameters using the learned distribution and then apply the randomly selected strategic parameter in the recommendation pipeline. 
A common recommender is visited by users hundreds of millions of times daily. Therefore, the proposed solution can be viewed as a probabilistic mixture over multiple strategic parameters. 
The probabilistic solution achieves supreme performance over deterministic solution since the latter can therefore be viewed as restricting the distribution vector in the probabilistic solution to a one-hot vector. 
We demonstrate the supreme performance using extensive simulation studies and online AB testing results later in this paper. 


The challenges of finding the optimal probabilistic parameters are two folded. On the one hand, the distribution of the unknown metrics can only be learned by interacting with the online customer-facing systems and then observing sparse and noisy samples via multiple iterations. 
On the other hand, the optimization problem to find the best probabilistic distribution falls into the family of stochastic compositional optimization, in which the unbiased stochastic gradient is unavailable. 
We propose to use an average of the unbiased estimator over the history in approximating the unbiased stochastic gradient. This trick also helped to address the sparse observation and reduced the noise. We also incorporated proximal projection to remove the simplex constraint. The proposed approach, Probabilistic pArameter optimization with unbiased STOchastic gradient approximation ($\ouralg$), converges at a rate $\mathcal{O}\left(1/{\sqrt{T}}  \right)$ and admits the regret in an order of $\mathcal{O}\left(\sqrt{T} \right)$.  

In sum, the key contributions of this paper are,
\vspace*{-1ex}
\begin{itemize}[leftmargin=*]
    \item  We discover a \emph{probabilistic} strategic parameter solution that outperforms the classic deterministic strategic parameter when we are pursuing multiple goals in recommendation systems.    
    \item  We formulate the problem of finding the optimal probabilistic parameter solution as a compositional stochastic optimization task, and developed an efficient stochastic gradient algorithm. We proved that the proposed algorithm converges to the optimal probabilistic distribution at a rate of  $\mathcal{O}\left(1/{\sqrt{T}}  \right)$ and the regret admits an order of $\mathcal{O}\left(\sqrt{T} \right)$ where $T$ is the number of iterations. 
    \item We implement the proposed probabilistic strategic parameter solution at a leading social network platform with hundreds of millions of daily active users and tens of  billions annual revenue. Note that in a platform of this scale, a slight percentage gain provides enormous business value. The proposed approach achieved $+0.22\%$ lift of user engagement in a recommendation scenario and $+1.7\%$ lift in revenue in an advertising optimization scenario compared with the optimal deterministic parameter choice. 
\end{itemize}

The rest of this paper is organized as follows. We first discuss related literature in Section~\ref{sec:related}. In Section~\ref{sec:problem} we formulate the strategic parameter searching as an optimization problem and argue why the probabilistic regime is better. We discussed our solution in Section~\ref{sec:algorithm} and provide the theoretical analysis in Section~\ref{sec:thm}. We include a series of simulation studies in Section~\ref{sec:simulation}, and present online AB testing results in Section~\ref{sec:online}. 

\section{Related Work}
\label{sec:related}
Black-box strategic parameter tuning has been extensively studied using both heuristic and Bayesian approaches to find the maximizer of an unknown and noisy objective function. Popular heuristic approaches such as Genetic Algorithm \citep{mitchell1998introduction},  Cross-Entropy-Methods \citep{rubinstein2013cross}, and Particle Swarm Optimization \citep{kennedy1995particle} have been widely adopted empirically. 
\citep{letham2019constrained,kandasamy2018parallelised} proposed to use the Bayesian approach to sequentially explore the strategic parameters.  \citep{hill2017efficient,ding2019whole, agarwal2015constrained} proposed to use the Bayesian approach with contextual information in industrial recommendation systems. Multi-armed bandits (MAB) and bandit optimization problems are also widely used in searching the best deterministic strategic parameter \citep{thompson1933likelihood,bubeck2012regret}. The contextual bandit setting considers environmental conditions and can select a deterministic parameter for each different context \citep{he2020contextual,krause2011contextual}. We note that this is different than the proposed probabilistic regime (see Section~\ref{sec:problem} for detailed discussion). 

Multi-task learning has raised attention recently in recommendation systems \citep{zhang2019deep, gomez2015netflix}. However, the focus of existing literature is in the phase of predicting immediate user actions \citep{gu2020deep, tang2020progressive, di2017adaptive,jugovac2017efficient} and not the strategic parameters. Different architecture for model-parameter sharing, loss sharing, etc., have been explored \citep{ma2018entire, zhao2019recommending, di2017adaptive, tang2020progressive}. Other ways of addressing multiple losses have also been studied, such as weighted sum \citep{lin2019pareto, ribeiro2014multiobjective} and adaptive optimization optimization\citep{di2017adaptive, tang2020progressive}.
Regarding learning deterministic strategic parameters with multiple objectives, \citep{cui2017novel} uses an evolutionary algorithm to include a diversity indicator on top of item rating evaluation; \citep{pei2019value} maximizes a so-called economic value based on reinforcement learning. Recently, much attention has been paid to finding the Pareto frontier of multiple goals \citep{milojkovic2019multi}, where a set of Pareto optimal items is selected, and no alternative can improve every objective simultaneously. And Pareto efficient algorithms can help coordinates multiple objectives \citep{lin2019pareto,ribeiro2014multiobjective} and the problem can be solve by entropy search \citep{suzuki2020multi,belakaria2019max}, expected hyper-volume improvement \citep{daulton2020differentiable}, etc. 

Recently, \citep{tu2020personalized} proposed to personalize the strategic parameter assignment based on the user's demographic group or device type.  However, in their proposed settings, deterministic single-best parameters are still the solution for each user group or context. We also note that \citep{tu2020personalized} uses post AB-testing data to estimate the stochastic effect of treatment/parameters. Our paper considered a dynamic and iterative approach that is more effective in industrial applications. 

\section{Why Probabilistic Solution is Better Than Singleton Solution}
\label{sec:problem}
%
Now we formally define the probabilistic strategic parameter solution and show its advantage over the deterministic one. We start by formulating the strategic parameter selection problem into an optimization task. Then we show that the proposed probabilistic solution defines a larger feasible domain than the deterministic solution, implying its superior performance. 
%
\subsection{Optimization View of the Strategic Parameter Tuning Problem}
Recall the motivating example in Section~\ref{sec:intro} where we tune the strategic parameter $\bm\beta$ to optimize {one or} multiple business goals, e.g., 
\begin{itemize}[leftmargin=*]
\item maximize the expected  averaged engagement time ($\mu^{X}$),
\item maximize the expected  averaged engagement time ($\mu^{X}$) while the like rate ($\mu^{Y}$) does not drop,
\item maximize a utility function of both the expected  averaged engagement time  ($\mu^{X}$) and the  like rate  ($\mu^{Y}$).
\end{itemize}
In this example, to be more precise, we denote by $\mu^X(\bm\beta)$ the expected daily averaged engagement time when applying the strategic parameter $\bm\beta$ in the system, and $\mu^Y(\bm\beta)$ follow a similar definition. Note that the ground truth value of $\mu^X(\bm\beta)$ or $\mu^Y(\bm\beta)$ is difficult to obtain. We can usually obtain an unbiased sample (denote as $\hat{u}^X(\bm\beta)$) by applying ${\bm \beta}$ to a group of randomly selected users, namely,
\[
\hat{u}^X(\bm\beta) := \frac{1}{\#\text{users}}\sum_{j} \text{user $j$'s $X$ metric (i.e., daily engagement time) applying strategic parameter $\bm\beta$}.
\]
Here $\mathbb{E}(\hat{u}^X(\bm\beta)) = \mu^X(\bm \beta)$ is a unbiased observation of $\mu^X(\bm \beta)$ whose variance depends on the number of users/requests in our observation. 
In general, we use $X$ to denote the metrics we primarily want to improve and $Y$ to denote the guardrail metrics we do not want to drop. 
Next we can formulate the \textbf{conventional} strategic parameter optimization in the form of a numerical optimization problem
\begin{align}
\label{eq:formulation1}
    \max_{\beta \in \mathcal{B}} \quad f(\mu^X({\bm \beta}), \mu^Y({\bm \beta}))
\end{align}
where $\mathcal{B}$ is the set of all possible strategic parameters. For different goals, $f$ can be designed as,
\begin{itemize}[leftmargin=*]
\item $f(x, y) = x$ if we only have one metric to optimize. This is less common in practice. 
\item $f(x,y) = x - h(y;c)$ where $h(y)=\infty $ if $y\leq c$ and $h(y)=0$ otherwise. This $f$ maximizes the metric $\mu^{X}$ while strictly requiring metric $\mu^{Y}$ to be higher than a lower-bound $c$. 
\item $f(x, y) = x - \lambda \min(0, y-c)^2$ and $\lambda>0$ is some constant. This combined utility function aims to improve metric ${X}$ while imposing a square penalty if the metric ${Y}$ drops below a pre-define threshold $c$. 
\end{itemize}
%
%
\subsection{Probabilistic Strategic Parameter Solution}
%
Now to illustrate our proposed probabilistic solution, we first reformulate \eqref{eq:formulation1} into an mathematically equivalent form. Here we assume that the number of possible strategic parameter options $\mathcal{B} = \{ \bm\beta^{1}, \ldots, \bm\beta^{K} \}$ is finite and is of size $K$ for simplicity. The infinite scenario follows the same spirit.
\footnote{See appendix for discussion on continuous parameter space. }
Let's first denote the rewards of different strategic parameters compactly in vector form
\[
 \bm\mu^{X}(\mathcal{B})= \left[\mu^{X}(\bm\beta^{1}), \ldots, \mu^{X}(\bm\beta^{K}) \right] \quad  \bm\mu^{Y}(\mathcal{B})= \left[\mu^{Y}(\bm\beta^{1}), \ldots, \mu^{Y}(\bm\beta^{K}) \right]
 \]
And re-write the optimization Eq~\eqref{eq:formulation1} in an equivalent form,
\begin{equation}
\begin{aligned}
\label{eq:formulation_single}
    \max_{{\bm p}} \quad & f(\bm\mu^{X}(\mathcal{B})\bp, \bm\mu^{Y}(\mathcal{B})\bp)
    \\ \text{s.t.} \quad & \bp \in S:=\left\{ \bp \in \{0,1\}^K | \Sigma_{k=1}^{K} p_k =1\right\}. 
\end{aligned}
\end{equation}
It is worth noting that the optimal solution to Eq.~\eqref{eq:formulation_single} selects the best option from $\mathcal{B}$ to optimize the target goal. We refer to this as the \textbf{deterministic} solution since the same strategy parameter is applied to all the recommendation requests. 

Next we are ready to propose the probabilistic solution. Revisiting ~\eqref{eq:formulation_single}, one can mathematically change feasible set selection variable $\bp$ to a larger set, $\bar{S} = \{ \bp\in \left[0,1\right]^{K} | \Sigma_{k=1}^{K} p_k =1\}$, which is the convex hull formed by $S$. $\bar{S}$ is the $K$-dimensional probability simplex and $\bp\in\bar{S}$ can be viewed as a pmf over $K$ strategic parameters in $\mathcal{B}$. 
%

The pmf view of $\bp$ suggests a new \textbf{probabilistic} solution of applying strategic parameters in recommendation systems. Say for example $K=3$ and $\bp = [0.8, 0.2, 0]$. When a user visits the platform and a recommendation request is created, we randomly select either  $\bm\beta^{1}$ with probability $0.8$ or $\bm\beta^{2}$ with $0.2$, and then apply the selected $\bm\beta^{1}$ or $\bm\beta^{2}$ as the strategic parameter to generate final recommendation list for the user. If we consider the average effect over hundreds of millions of daily requests in a industrial recommender, the inner product $\bm\mu^{X}(\mathcal{B})\bp = \sum_k p_k \mu^{X}(\bm\beta^{k})$ still represents the expected metric $X$ (here the expectation is over both the randomness of sampling from $\bp$ and the noise in $\hat{u}^{X}$).  
%
%
%
%
%
Formally, we propose to pursue the optimal probabilistic solution by solving,
\begin{equation}
\begin{aligned}
\label{eq:formulation_prob}
    \max_{{\bm p}} \quad & f(\bm\mu^{X}(\mathcal{B})\bp, \bm\mu^{Y}(\mathcal{B})\bp)
    \\ \text{s.t.} \quad & \bp \in \bar{S} = \{ \bp\in \left[0,1\right]^{K} | \Sigma_{k=1}^{K} p_k =1\}.
\end{aligned}
\end{equation}
%
\subsection{Probabilistic Solution Achieves better Objective Value than Singleton Solution}
Since $f$ is in general a non-linear in Eq~\eqref{eq:formulation_prob}, the optimal $\bp$ to Eq~\eqref{eq:formulation_prob} may not be a one-hot solution. For the same $f$ and $\bm\mu$'s in ~\eqref{eq:formulation_single} and ~\eqref{eq:formulation_prob},  let $\bp_S^\star$ be the optimal (deterministic parameter) solution to Eq.~\eqref{eq:formulation_single}  and  $\bp_{\bar{S}}^\star$ be the optimal (probabilistic parameter) solution to Eq.~\eqref{eq:formulation_prob}. Since the feasible set of ~\eqref{eq:formulation_single} (deterministic solution) is a subset of ~\eqref{eq:formulation_prob}. It is straightforward to verify that,

\noindent\textbf{Observation 1} The optimal objective value  $f(\bm\mu^{X}\bp_S^*, \bm\mu^{Y}\bp_S^*) \leq f(\bm\mu^{X}\bp_{\bar{S}}^*, \bm\mu^{Y}\bp_{\bar{S}}^*))$. Namely, \textit{ the probabilistic parameter solution achieves better reward compared to the deterministic parameter}. 

\begin{figure}[!hbt] 
\centerline{\includegraphics[width=0.4\textwidth]{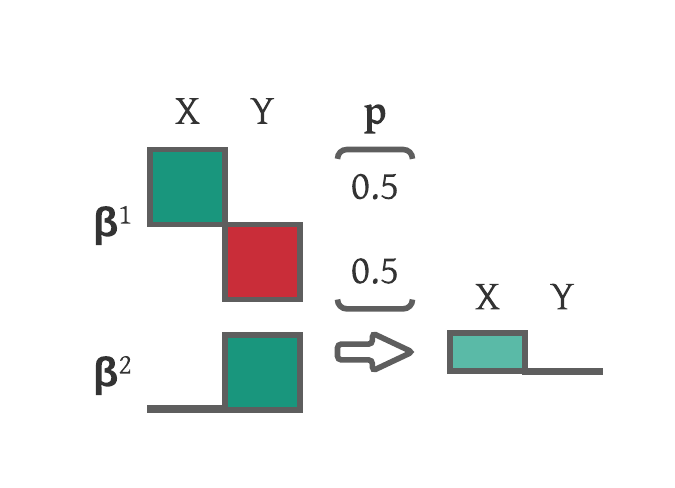}}
\caption{\small An Illustrating Example on why probabilistic parameter can achieve better overall rewards than single deterministic parameter. $X,Y$ refer to two different business metrics, $\bm\beta^{1}, \bm\beta^{2}$ are two possible strategic parameter choices. $\bp=[0.5, 0.5]$ is a probabilistic parameter mixing weight that samples $\bm\beta^{1}$ or $\bm\beta^{2}$ with equal probability for each request. The right-hand side indicated the overall rewards achieved with $\bp$.  }
\label{fig:toy} 
\end{figure}

\textbf{An Illustrating Example}: To illustrate Observation 1, consider $K=2$ strategic parameters and $\mu^{X}(\mathcal{B}) = [2, 0]$,i.e., the first choice $\bm\beta^{1}$ has an expected engagement time of $2$ and the engagement time is $0$ for $\bm\beta^{2}$). In practice we focus more on the relative gain over some baseline. So setting $\mu^{X}(\bm\beta^{2})=0$ means there is no relative gain on engagement time over baseline. We also set $\mu^{Y}(\mathcal{B}) = [-2, 2]$ for metric $Y$. Let the objecive function be $f(x,y) = x -h(y,c=0) $ where $h$ imposes infinity penalty when $y<0$. For deterministic solutions, 
selecting first parameter $\bp_{(1)} = [1,0]^{\top}$ results in $f = -\infty$ since the metric $Y$ is below the threshold $c=0$ for first strategic parameter $\bm\beta^{1}$. 
On the other hand, selecting the second parameter $\bp_{(2)} = [0,1]^{\top}$ result in $f = 0$ since $\bm\beta^{2}$ has a lower metric $Y$. Now for the probabilistic case, we can choose  $\bp_{(\text{mix})} = [0.5, 0.5]^{\top}$ and the average metric $\bm\mu^{Y}\bp = 0$ is above the threshold so the reward is $f = 1$. Therefore, the $\bp_{\text{(mix)}}$ achieves better objective value than the best deterministic solution $\bp_{(2)}$.
 
The illustrating example we've discussed represents a common real-world scenario, that is, one strategic parameter choice ($\bm\beta^{1}$ in this example) improves metric $X$ but at the same time results in a lower average value of metric $Y$, and vice versa for another strategic parameter choice ($\bm\beta^{2}$). This scenario is common as multiple business metrics in real applications often compete with each other (given limited user attention on the platform). We will demonstrate this in Section~\ref{sec:simulation} and~\ref{sec:online}. 

Before concluding this section, note that we used two metrics $X,Y$ so far for simplicity. In real application, there are typically multiple metrics $X$'s to improve and various metrics $Y$'s to protect. The formulation and conclusions in this section directly extends to this general setting.

\section{How to Solve the Probabilistic Parameter Optimization}
\label{sec:algorithm}
We now discuss how to solve the probabilistic parameter optimization problem in Eq~\eqref{eq:formulation_prob}, which falls into the family of constrained stochastic compositional optimization \citep{wang2016accelerating} tasks. The key difficulty in solving  Eq~\eqref{eq:formulation_prob} is that the ground-truth rewards ${\bm\mu}^{X}({\bm \beta})$ and ${\bm\mu}^{Y}({\bm \beta})$ are unknown. 
To obtain an empirical sample $\hat{u}$ in practice, one needs to choose one of the strategic parameters (say $\bm\beta^{k}$), apply it to the actual recommender system (typically to a small group of users/requests), and observe the empirical average metrics $\hat{u}^{X}(\bm\beta^{k}), \hat{u}^{Y}(\bm\beta^{k})$ (in our motivating example, they are the user's daily engagement time and his/her like rate), and repeat this for multiple rounds/iterations. 
The observation at each iteration is a sparse sample of ${\bm\mu}^{X}({\bm \beta})$ and ${\bm\mu}^{Y}({\bm \beta})$. 
Furthermore, the observation of one iteration can only be collected after a certain time. Say, if hourly engagement time is our goal, then, we need to apply $\bm\beta^{k}$ in the system for at least one hour. 

We choose to solve our problem using stochastic gradient approach. The sparse and noisy observation in our problem raises two technical challenges: 1) the unbiased stochastic gradient is generally hard to obtain; 2) the constraint is quite tricky to enforce. We next discuss our technical solutions to address them. 
%
\subsection{Unbiased Stochastic Gradient Approximation}
%
In solving Eq~\eqref{eq:formulation_prob}, an unbiased stochastic gradient for is hard to obtain since the our observation in each round to approximate  the true expectations ${\bm\mu}^{X}({\bm \beta})$ and ${\bm\mu}^{Y}({\bm \beta})$ are noisy and sparse. To be concrete, let's
define in the matrix form
\begin{equation}
\label{eq:matrix_notation}
\bm\mu = \bm\mu(\mathcal{B}) = 
\begin{bmatrix} 
\mu^{X}(\bm\beta^{1}), \ldots, \mu^{X}(\bm\beta^{K}) \\
\mu^{Y}(\bm\beta^{1}), \ldots, \mu^{Y}(\bm\beta^{K})
\end{bmatrix} 
\in \mathbb{R}^{2\times K}
\end{equation}
and compactly write $f(\bm\mu \bp) $ where $\bm\mu \bp = \left[ \bm\mu^{X}(\mathcal{B}) \bp, \bm\mu^{Y}(\mathcal{B}) \bp \right]^{\top} $ for the objective function in Eq~\eqref{eq:formulation_prob}. Note that there could be more than 2 rows in $\bm\mu$ if we have more than 2 metrics of interests. For simplicity, we will keep the dimension as 2 throughout this section.  

Since vector $\bm\mu$ are the expectation of the noisy metric observations,  the optimization objective $f(\bm\mu \bp)$ in ~\eqref{eq:formulation_prob} is different from the common stochastic optimization admitting the form of $\E \left[f(\cdot)\right]$ (with no expectation inside $f$).
This falls into the family of stochastic compositional objectives and the unbiased stochastic gradient is generally not achievable \citep{wang2017stochastic}. More specifically, the true gradient of~\eqref{eq:formulation_prob} at learning step $t$ admits the form
\[
 \frac{\textbf{d} f(\bm\mu \bp)}{\textbf{d} \bp} = \bm\mu^\top \nabla f(\bm\mu \bp),
\]
which requires us to obtain unbiased estimator of the term $\bm\mu^{\top} \nabla f(\bm\mu \bp)$ at each iteration (learning step) $t$. Recall that at each round $t$, we can sample one strategic parameter $\bm\beta^{i_t}$ from the current distribution $\bp_{t}$, and observed empirical metrics $\hat{u}^{X}(\beta^{i_t}), \hat{u}^{Y}(\beta^{i_t})$, etc. We can then   construct
\begin{equation}
\begin{aligned} 
\label{eq:Uhat}
\hat{U}_t =   \begin{bmatrix} 
0,& \ldots &, 0, \hat{u}^{X}(\beta^{i_t}) / p_{t, i_t},  0, \ldots, 0 \\
0,& \ldots &, 0, \hat{u}^{Y}(\beta^{i_t}) / p_{t, i_t},  0, \ldots, 0 \\
\end{bmatrix} 
\end{aligned}
\end{equation}
and verify that $\E( \hat{U}_t ) = \bm\mu$ is unbiased estimator of $\bm\mu$. However,
this cannot ensure that the stochastic gradient we compute is unbiased, because
$\E ( \hat{U}_t^{\top} \nabla f(\hat{U}_t \bp_t) ) \neq \E(\hat{U}_t )^\top \nabla f(\E (\hat{U}_t) \bp_t) = {\bm\mu}^\top \nabla f(\bm\mu \bp_t)$
unless $f$ is linear. In this case, our solution is to use the averaged value of all historical (unbiased) estimator
\begin{align}
\label{eq:estimator}
\hat{V}_t:=\frac{1}{t}\sum_{s=1}^t\hat{U}_s
\end{align}
%
which gets more and more accurate as training continues. We then approximate the unbiased gradient using 
\begin{equation}
\label{eq:gt}
    \bm{g}_t = \hat{V}_t^\top \nabla f(\hat{V}_t \bp_t).
\end{equation} 
We note that our gradient estimation is different than the stochastic compositional gradient descent (SCGD) framework in \citep{wang2016accelerating, wang2017stochastic}. This difference is due to the linear form $\bm\mu\bp$ inside the $f$ function, a special case of the generic SCGD. Since we used the history average $\hat{V}_t$, the estimation of gradient in ~\eqref{eq:gt} is more stable and converges faster than the generic SCGD solution \citep{wang2017stochastic} (see Section~\ref{sec:simulation} for simulation study).

\subsection{KL divergence removes the simplex constraint}
The second technical challenge is to handle the simplex constraint $\bp \in \bar{S}$. One can certainly apply the projected step after the (approximate) stochastic gradient descent step. We choose KL divergence (other than the Euclidean distance) as the Bregman distance which can naturally ensure the next iterate $\bp_{t+1}$ within the simplex constraint $\bar{S}$ even without considering the simplex constraint. More specifically, $\bp_t$ is updated by the following constraint-free proximal step with a given step-size $\gamma_t$
\vspace*{-1ex}
\begin{equation*}
\bp_{t+1} = \text{argmin}_{\bp} - \langle {\bm g}_t, \bp \rangle + \frac{1}{\gamma_t}\text{KL}(\bp \| \bp_t)
\end{equation*}
\vspace*{-1ex}
resulting in the following closed form update rule,
\begin{equation*}
w_{t,k} = p_{t,k} \exp\left( \gamma_t \bm g_{t,k} \right),~k=1,\ldots, K, ~\quad~ p_{t+1,k} = {w}_{t, k} / \sum_{k'=1}^{K} {w_{t,k'}} .
\end{equation*}
%
\subsection{Overall Algorithm}
%
We summarize the proposed robabilistic Parameter Optimization with Unbiased Stochastic Gradient Approximation in Algorithm~\ref{alg:TS-SCGD} as an iterative exploration and optimization procedure. At each round, we first sample one strategic parameter from the current probabilistic pmf $\bp_t$, observed its rewards from online recommender, and then conduct the gradient optimization. Note that instead of using the $\bp_T$ from last round as the final return, we used the averages $\bar{\bp}_{T} = {1\over T}\sum_{s=1}^{T} \bp_s$ as the estimated best probabilistic solution, which is more stable in practice. We show this average do converges to the optimal rewards. 
\begin{algorithm}[!hbt]
\caption{\small {\textbf{P}robabilistic p\textbf{A}rameter Optimization with unbiased \textbf{STO}chastic Gradient Approximation} (PASTO)} 
\begin{algorithmic}[1]
\REQUIRE Initial estimator of rewards $\widehat{U}_0$ ($2\times K$ matrix), learning rate $\gamma$, smoothing parameter $\epsilon_{t}, t=1,\ldots, T$, the total number of iteration steps $T$;
\ENSURE  $\bar{\bp}_T$
\STATE Initialize $\mathbf{w}_{0} = \left[ 1,\ldots ,1 \right]^{\top}=:\mathbf{1} \in \mathbb{R}^{K}$, $\hat{V}_0 \leftarrow \hat{U}_0$
\FOR {each iteration $t=1, \ldots, T$}
\STATE Sample one strategic parameter  $\bm\beta^{i_t}$ from $\mathcal{B}$ based on the pmf $\bm{p}_t = (1-\epsilon_t) \frac{\bm{w}_{t-1}}{\|\bm{w}_{t-1}\|_1} + \epsilon_t/K \cdot \mathbf{1}$
\STATE Apply the sampled strategic parameter $\bm\beta^{i_t}$ to the online system and observe $\hat{u}^{X}({\bm\beta}^{i_t}), \hat{u}^{Y}({\bm\beta}^{i_t}), \ldots$. Construct $\hat{U}_t$ based on Eq~\eqref{eq:Uhat}
\STATE Update $\hat{V}_{t} = \frac{t}{t+1} \hat{V}_{t-1} + \frac{1}{t+1}  \hat{U}_t $ (This is the same as Eq~\eqref{eq:estimator} but in a recursive form.)
\STATE Compute the stochastic gradient using Eq~\eqref{eq:gt} ${\bm g}_t:=   \left(\hat {V}_t\right)^{\top}\nabla f(\hat{V}_t\bm p_t)$
\STATE Update $w_{t,i} = w_{t-1,i}\exp\left(\gamma g_{t,i} \right) $
\ENDFOR
\STATE Return $\bar{\bp}_{T} := \frac{1}{T}\sum_{t=1}^{T} \bp_t$.
\end{algorithmic}
\label{alg:TS-SCGD}
\end{algorithm}

 Note that following \citep{exp3}, we added a smoothing parameter $\epsilon_t$ at each iteration to ensure all candidate parameters have a non-zero chance of being selected for exploration and to cap the unbiased estimation ~\eqref{eq:Uhat} . Intuitively, higher $\epsilon_t$ also introduces more exploration of the strategic parameters. Similarly, the step-size $\gamma$ also controls the degree of exploration.

\section{Theoretical Analysis}
\label{sec:thm}
We now present the theoretical analysis for Algorithm~\ref{alg:TS-SCGD}. Without loss of generality, we assume that $f$ in the optimization problem~\eqref{eq:formulation_prob} is a concave objective function. 
\footnote{Note that $f$ represents the rewards to be maximized in application. Therefore, we assume $f$ being concave in our analysis.  } 
For the results to hold, we first introduce a few common technical conditions on our objective function $f$ and the underlying ground-truth rewards $\bm\mu$'s.
\begin{assumption}\label{ass}
We make the following commonly used assumptions for $f(\cdot)$ and $\bm\mu(\mathcal{B})$:
\begin{itemize}
\item  The gradient of objective $f(\cdot)$ defined in Eq~\eqref{eq:formulation_prob} is bounded, and $f(\cdot)$ is  $L_f$-smooth, i.e.,
\begin{align*}
\Vert \nabla f(\bm{\theta}_1) - \nabla f(\bm{\theta}_2) \Vert \leq   L_f \Vert \bm{\theta}_1 - \bm{\theta}_2 \Vert ~\quad\text{and}~ \|\nabla f(\bm{\theta}_1)\|\leq  G_f,\quad \forall \bm{\theta}_1,~\bm{\theta}_2
\end{align*}
%
\item The Frobenius norm of $\bm\mu(\mathcal{B})$ defined in Eq~\eqref{eq:matrix_notation} is bounded, i.e., $\|\bm\mu(\mathcal{B})\|_F \leq G_U$
\item The magnitude of the gradient estimation $\bm g_t$ is bounded for all iterations, i.e., $\|\bm g_t\|\leq G, \forall t.$
\end{itemize}
\end{assumption}

\subsection{Convergence Analysis} 
Since $\bar{\bp}_T$ is what we deploy in the actual production system after the learning process, our \textbf{primary} focus here is the  convergence of  Algorithm~\ref{alg:TS-SCGD}. Namely, we want to show the estimate $\bar{\bp}_T$ do  converge to the optimal probabilistic mixing vector as $T$ goes to infinity. 
Equivalently, we need to show that the overall reward that can be achieved by $\bar{\bp}_T$, i.e., $f\left(\bm\mu\overline{\bm{p}}_T  \right)$, can converge to the optimal reward of Eq~\eqref{eq:formulation_prob} (the convergence is in the sense of expectation). Formally, 

\begin{theorem}\label{main:theorem_convex_convergence} (\textbf{Convergence of Alg.~\ref{alg:TS-SCGD}})
Under Assumption~\ref{ass},  set $\epsilon_t = \frac{G}{\sqrt{t}}$ and $\gamma = \frac{1}{\sqrt{T}}$, further, assuming that the noisy observations $\hat{U}_t$ and $\hat{U}_s$ from two different iterations $t,s$ are independent. Then, 
\begin{align*}
\max_{\bp \in\bar{S}}  f\left( \bm\mu \bp  \right) - \E\left[f\left(\bm\mu\overline{\bm{p}}_T  \right)\right] 
\leq & \frac{1}{\sqrt{T}}\left(\ln K + 4eG^2  +\frac{32(G_f^2 + G_U^2L_f^2)G_U^2}{G}\right).
\end{align*}
Here the expectation is over all the randomness in the gradient history $\{\bp_t \}_{t=1}^{T}$.
\end{theorem}
In short, the estimation value $f(\bm\mu\bar{\bp}_{T})$ converges to the optimal reward in expectation and the convergence rate admits the order of  $\mathcal{O}\left(\frac{\ln K}{\sqrt{T}} \right)$. Theorem~\ref{main:theorem_convex_convergence} implies that using the $\bar{\bp}_T$ in the online production system can yield the optimal rewards of the probabilistic solution. 

\subsection{Regret Analysis}
Next, we analyze the regrets bound of the iterative learning process. In this analysis, we assume the ground-truth rewards $\bm\mu$ could be different over time. This is often the case in real-world applications since customer preferences and overall trends are constantly changing.
Concretely, we denote by $\bm\mu_{t} = \bm\mu_{t}(\mathcal{B})$ the ground-truth metric at $t$-th iteration (the same format as in Eq~\eqref{eq:matrix_notation}).

From a practice viewpoint, till $t$-th round, the average historical metrics would be  $\frac{1}{t}\sum_{s=1}^{t}\bm\mu_{s}$
in expectation and $\hat{V}_t = \frac{1}{t} \sum_{s=1}^{t}\hat{U}_s$ the empirical observation. Given this, $f(\frac{1}{t}\sum_{s=1}^{t}\bm\mu_{s} \bp) $ represents the (ideal) objective value if one can re-select a solution $\bp$ at $t$-th interation given the full history of observations, and the regret can be defined as the gap between this reward objective with $\bp_t$ and the global optimal $\bp^{*}$, $f(\frac{1}{t}\sum_{s=1}^{t}\bm\mu_{s} \bp^{*}) - f(\frac{1}{t}\sum_{s=1}^{t}\bm\mu_{s} \bp_{t})$, where $\bp_t$ is from Algorithm~\ref{alg:TS-SCGD}. Similarly, the regret can be also defined as the gap $f(\frac{1}{t}\sum_{s=1}^{t}\bm\mu_{s} \bp^{*}) - f(\frac{1}{t}\sum_{s=1}^{t}\hat{V}_{s} \bp_{t})$ between the global optimal and the empirical loss, which is more close to the empirical regrets in the procedure of Algorithm~\ref{alg:TS-SCGD}. 

With all these notations, we show the following bounds on the total regret, 
\begin{theorem}\label{main:theorem_convex} (\textbf{Regret of Alg.~\ref{alg:TS-SCGD}})
For Algorithm~\ref{alg:TS-SCGD}, under Assumption~\ref{ass}, $\epsilon_t = \frac{G}{\sqrt{t}}$ and $\gamma = \frac{1}{\sqrt{T}}$, the regrets are:
\begin{align*}
&\max_{\bp \in \bar{S} } \sum_{t=1}^T f\left(\frac{1}{t}\sum_{s=1}^t \bm\mu_s\bp \right) - \sum_{t=1}^T\E\left[f\left(\frac{1}{t}\sum_{s=1}^t \bm\mu_s \bp_t  \right)\right]
\leq  \sqrt{T}\left(\ln K + 4eG^2  +\frac{32(G_f^2 + G_U^2L_f^2)G_U^2}{G}\right).
\end{align*}
and
\begin{align*}
&\max_{\bp \in \bar{S} } \sum_{t=1}^T f\left(\frac{1}{t}\sum_{s=1}^t \bm\mu_s\bp \right) - \sum_{t=1}^T\E f\left(\frac{1}{t}\sum_{s=1}^t \hat{U}_s \bp_t  \right) \leq \sqrt{T} \left(\ln K + 4eG^2 +\frac{32(G_f^2 + G_U^2L_f^2)G_U^2}{G} + 2G_U^{2}\right).
\end{align*} 
\end{theorem} 
As Theorem~\ref{main:theorem_convex} suggests, our Algorithm~\ref{alg:TS-SCGD} can achieve a regret at the order of  $\mathcal{O}(\sqrt{T})$. We defer all the proof details in supplementary.

\section{Simulation Study}
\label{sec:simulation}
We conduct simulation to demonstrate the \textbf{gain} of probabilistic solution over single deterministic ones. For all experiments, we choose objective of the form $f(x, y_1, y_2, \ldots) = x - 5.0*\min(0,  y_1-c_1 )^{2} - 5.0*\min(0,  y_2-c_2 )^{2} - \ldots$ if there are multiple guardrail metrics $y$s. For Algorithm~\ref{alg:TS-SCGD}, we set $\epsilon_t = 0.1/\sqrt{t+10}$ and $\gamma = 0.1/K$.  Since our goal is to compare the performance of probabilistic solution against the deterministic one. Therefore, when reporting the performance of the  deterministic(single) solution, we always assume having access to the noise-less rewards and knowing which single strategic parameter yields the highest expected objective value, which will be the upper bound for \emph{any} deterministic parameter algorithm.

\noindent\textit{Parallel Querying}: in each round of Algorithm~\ref{alg:TS-SCGD}, we could choose $ Q\geq 1$ choices of strategic parameters by sampling from $\bp_t$. Practically,  we can randomly split the users or requests into $Q$ subgroups and apply one different strategic parameter for each group accordingly. 
%
\subsection{Simulation on Synthetic Data}
\label{sec:validation}
We first verify that the proposed Algorithm~\ref{alg:TS-SCGD} does converge empirically using the following setup:
\begin{enumerate}[leftmargin=*]
\item[(A)] the illustrating example in Section~\ref{sec:problem} where  $K=2$, $\bm\mu^{X} = [2, 0], \bm\mu^{Y} = [-2, 2]$  and $c=0$. If the $k$-th paramter is queried in one round, we add a Gaussian noise of $\mathcal{N}(0, 5.0)$ to $\mu^{X}_k$ (and $\mu^{Y}_k$ resp.) to simulate the noisy observations $\hat{u}^{X}_k$ and $\hat{u}^{Y}_k$. We set parallel querying $Q=1$.
\item[(B)] $K=100$ and $\mu^{X}_k$, $\mu^{Y_1}_k$,  $\mu^{Y_2}_k$ for $k=1,\ldots,K$ are sampled independently and uniformly within $[-1, 1]$. We set $c^{Y_1}=c^{Y_2} = 0.5$. An additive Gaussian noise $\mathcal{N}(0, \sigma^{2})$ is added to each round's observation if a strategic parameter is selected. We set parallel querying $Q=10$.
\end{enumerate}

\begin{figure}[!hbt] 
\centerline{\includegraphics[width=0.5\textwidth]{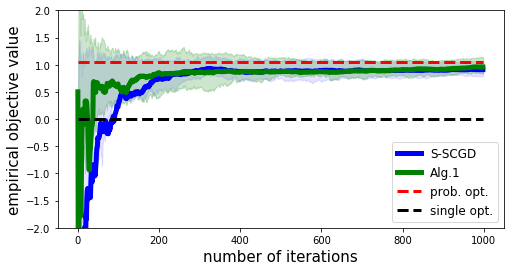}}
\caption{\small{ Empirical convergence of our proposed Algorithm 1 and S-SCGD on the example A. Shaded area indicate the 75th and 25th percentile of the objective in each round in $1000$ Monte Carlo runs.}}
\label{fig:convergence} 
\vspace*{-1ex}
\end{figure}
For the first toy example (A), we conduct $1000$ Monte Carlo runs and report the average empirical objective value. To evaluate the efficiency of the proposed approach,  we also compared against the stochastic compositional gradient descent approach \citep{wang2017stochastic} (S-SCGD). 
%
As shown in Figure~\ref{fig:convergence}, our proposed algorithm does converge to the correct optimal probabilistic solution. As one can easily verify from the simulation setting, this probabilistic reward is indeed better than the single parameter solution. We also note that the convergence of Algorithm 1 is faster than the simple S-SCGD due to its reduction of noise. We will only use proposed Algorithm~1 in later sections due to its efficiency and robustness to sparse observations. 

%
%
\begin{figure}[!hbt]
\centerline{\includegraphics[width=0.5\textwidth]{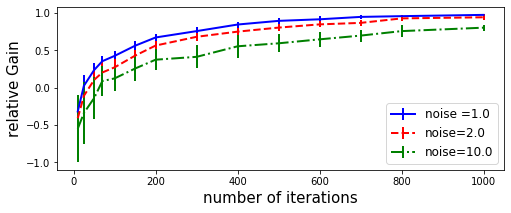}}
\caption{\small{ Relative gain of Algorithm~1 on synthetic dataset with different noise level $\sigma$. Average and standard deviation of 1000 random runs are reported.}}
\label{fig:convergencenoise} 
\vspace*{-1ex}
\end{figure}
For the second simulation setting (B), we also conduct 1000 Monte Carlo runs to generate ground-truth $\bm\mu$ and simulate our algorithm. To make it comparable across different ground-truth setting $\bm\mu$, we measure the relative gain of the probabilistic parameter over the single best parameter, defined as  $r:=\frac{|f(\bar{\bp}_t) - f(\bp_{S}^{*}) |}{|f({\bp}_{\bar{S}}^{*}) - f(\bp_{S}^{*}) |}$, i.e., the gain over single best parameter normalized by the ideal objective gap between probabilistic parameter and single parameter. $r>0$ indicates achieving a better objective than the single-best parameter. 
%
%
Figure~\ref{fig:convergencenoise} summarizes the average relative gain as the number of iteration for different noise levels. The error bars indicate the standard deviation of the 1000 Monte Carlo runs. We note that \texttt{TASCO} does converge and approaches the best possible gain. The noise of observation does impact the convergence speed of the approach. 
%
\subsection{Simulation on Real Dataset}
Motivated by \citep{milojkovic2019multi,ma2018entire}, we simulate on two publicly available datasets, Amazon Books \citep{ni2019justifying} and MovieLens 20M\footnote{\url{https://grouplens.org/datasets/movielens/20m/}}. Each of them has customer-item rating interactions as well as price, genre information.
We binaries the 5-star rating and set ratings $\geq 3$ as positive. Users and items with at least 5 ratings are kept in the processed data. Data were then split into training and test with $70\%, 30\%$.

We considered three targeted metrics: $1)$ the accuracy of relevance and measure it with the standard Recall@K (R@K) metrics, $2)$ the revenue achieved by ranking measured by the recall metrics weighted by the actual price of item (denoted as Revenue@K or REV@K) (for Amazon dataset),  and $3)$ the accuracy of predicting a specific category video of `Documentary' genre (for Movielens 20M dataset) (R-D@K). Higher metrics indicate better performance. We build Variational Auto-Encoder (VAE) \citep{liang2018variational} with 2 hidden layers using the training set as the underlying predictive model for each of the 3 targets. All the setups are summarized in Table~\ref{table:settingreal}.

\vspace*{-1ex}
\begin{table}[!htb]
\caption{{Setup of real world data set}  }
\label{table:settingreal}
\begin{center}
\begin{tabular}{l|c|c|c|c|c}
\toprule
Data & obj.1 & obj.2 & \#.users & \#.items & \#.events \\
\midrule
Amazon & R@20 & REV@20 &  $93,976$ & $25,896$ & $964,363$\\
ML20M & R@20 & R-D@20 & $132,580$ & $8,936$ &$6,316,389$ \\
\bottomrule
\end{tabular}
\end{center}
\vspace*{-2ex}
\end{table}
The simulation runs as following. For each dataset, objective 1 in Table~\ref{table:settingreal} is viewed as $x$ and objective 2 as $y$. The underlying VAE models we have built predict the scores $p_1(\text{item, user}), p_2(\text{item, user})$, these scores are then aggregated using a power-based function 
$
s = p_1(\text{item, user})^{\alpha} p_2(\text{item, user})^{1-\alpha}
$
to rank the items. 
Here the parameter $\alpha$ as our strategic parameter.

We discretize the parameter space of $\alpha$ into $K=100$ choices and set $c$ being the metric of objective 2 when $\alpha = 0.5$. To simulated the noise in the online reward collection regime, for each round $t$, we randomly split the testing data into 10 folds,  apply $Q=10$ parallel strategic parameters, and then compute the empirical average objective metrics on these random subsets of testing data.

\begin{table}[!htb]
\caption{{Results on real world data simulation. Threshold is 4.28 for REV@20 an 0.080 for R-D@20. }}
\label{table:simres}
\begin{center}
\begin{tabular}{l|cc|cc}
\toprule
 & \multicolumn{2}{c|}{Amazon Book} & \multicolumn{2}{c}{ML20M} \\
&  R@20         & Rev@20         &  R@20         &    R-D@20 \\
\midrule
Probabilistic. &  $\bf{0.284} $  & ${4.31}$ & $\bf{0.424}$  & ${0.082}$    \\
Single best    &  $0.279$   & $4.31$ & $0.405$  & $0.084$    \\
Single goal    &  $0.286$   &   --    &   $0.430$  &    --     \\
\bottomrule
\end{tabular}
\end{center}
\end{table}
We report the final rewards and constraints metrics in Table~\ref{table:simres}. The single best parameter, as discussed before, is identified by iterating through all possible choices using the entire testing set. We also report the reward by having non-constrained single obj.1 to understand the upper limit of our prediction models. Table~\ref{table:simres} shows our probabilistic solution can achieve overall better results over the deterministic best parameter choices.
%
\section{Real-World Application}
\label{sec:online}
We present two industrial applications on a leading social networking platform with hundreds of millions of daily active users and the AB-testing results.  
\begin{table}[!hbt]
\caption{AB Test Result of Online Content Recommendation Task, Watch time is primary reward we would like to optimize. Like and Sharing are two guardrail constraints. A soft threshold of $-2.0\%$ was imposed on each of the constraint metrics. All results are reported as lift w.r.t. the baseline model at the time of experiment. }
\label{table:online1}
\begin{center}
\begin{tabular}{l|c|c|c}
\toprule
 & time & like  & sharing  \\
\midrule
Single Best & $+0.42\%$ & $-1.67\%$  & $-1.31\%$\\
Probabilistic &  $\bf{+0.64\%}$ & $\bf{+0.33\%}$ &  $\bf{+0.54\%}$ \\
\bottomrule
\end{tabular}
\end{center}
\vglue -2ex
\end{table}

\noindent \textbf{Ensemble Sort in Content Recommendation }
We consider the content ranking in one of the recommendation scenarios. The strategic parameters are ensemble weights of multiple recommendation queues, each optimized for a particular customer target events. This is similar to the illustrating example in Figure~\ref{fig:example}.  The goal is to increase the average user engagement (time) while not dropping the `like' action rate and the `share' action rate. The existing baseline parameter is a deterministic parameter and is extensively optimized. The AB testing results are outlined in Table~\ref{table:online1}. Note that the probabilistic setting improves not only the primary metrics but also the constraints metrics. This is due to the fact that some of the parameters in our probabilistic mix do yield higher gain in the `like' and `sharing' metrics. 

\noindent \textbf{Quota-Based Ads Retrieval Systems}
We also present real-world testing results in an ads-retrieval system of the platform. 
In this system, there are multiple modules, each attempting to retrieve a set of relevant ad contents with different types. One needs to combine these candidate sets into a single and smaller group to feed the downstream ad ranking and pricing models. Due to the limitation on computation time and power,  quotas need to be set on the maximum number of ad contents generated by each module. Our goal here is to improve the overall advertising revenue and not to hurt a number of specific categories such as cold-start ad content. 

We implement our probabilistic parameter solution in this ads system and compare it against an existing single-parameter baseline that has been extensively optimized.  In our online AB test, the algorithm achieves a revenue improvement of $+1.7\%$ without significantly hurting the imposed constraints. 

\section{Conclusion}
This paper argues that the probabilistic strategic parameter achieves better rewards than the deterministic parameter solution. We present an algorithm (\texttt{PASTO}) based on stochastic gradient descent to solve the probabilistic solution with theoretical guarantees.  Both simulation and online applications have shown improvement over the deterministic best arm choice.

\bibliography{ref1}  

\newpage 
\appendix
\section{Extension to Continuous Parameter Space}
We briefly discuss the extension of our formulation to the case when the parameter spaces are continuous. We first revisit the optimization problems in discrete space in Eq~\eqref{eq:formulation_single} and Eq~\eqref{eq:formulation_prob}. We rewrite the equations in the generic inner-product form as:

(for the deterministic parameter in discrete space)
\begin{equation*}
\begin{aligned}
    \max_{{\mathbf{ p }}} \quad & f\left(\langle \boldsymbol{\mu}^{X}, \textbf{p} \rangle, \langle \boldsymbol\mu^{Y}, \textbf{p}\rangle\right) = f\left( \sum_{\beta \in \mathcal{B}}\mu^{X}({\beta}) p(\beta),  \sum_{\beta \in \mathcal{B}}\mu^{Y}({\beta}) p(\beta) \right)
    \\ \text{s.t.} \quad & \textbf{p} \in S:=\left\{ \mathbf{p} ~| ~  p(\beta) \in \{0, 1\}, \forall \beta \in \mathcal{B}, \sum_{\beta \in \mathcal{B}} p(\beta) =1\right\}.
\end{aligned}
\end{equation*}

(for the probabilistic parameter in discrete space)
\begin{equation*}
\begin{aligned}
     \max_{{\mathbf{ p }}} \quad & f\left(\langle \boldsymbol\mu^{X}, \textbf{p} \rangle, \langle \boldsymbol\mu^{Y}, \textbf{p}\rangle \right) = f\left( \sum_{\beta \in \mathcal{B}}\mu^{X}({\beta}) p(\beta),  \sum_{\beta \in \mathcal{B}}\mu^{Y}({\beta}) p(\beta) \right)
    \\ \text{s.t.} \quad & \textbf{p} \in \bar{S} := \left\{ \mathbf{p} ~| ~p(\beta) \in [0, 1], \forall \beta \in \mathcal{B} , \sum_{\beta \in \mathcal{B}} p(\beta) =1\right\}
\end{aligned}
\end{equation*}

In the above formulation for the discrete case, the inner product $\langle \boldsymbol\mu^{X},\textbf{p}\rangle$ is between vectors $\boldsymbol\mu$ and $\mathbf{p}$ and can be expressed as a summation over a finite number of items. When $\mathcal{B}$ is a continuous space, we can similarly define $\boldsymbol\mu = \mu(\beta)$ as a real-valued function over $\beta \in \mathcal{B}$ and $\textbf{p} = {p}(\beta)$ as a PDF over $\mathcal{B}$. The above form of inner product  still hold but should be interpreted as the inner product between two real-valued functions, that is the summation of an infinite number of items or integration, specifically,

(for the deterministic parameter in continuous  space)
\begin{equation*}
\begin{aligned}
    \max_{\mathbf{p}} \quad & f\left(\langle \boldsymbol\mu^{X}, \textbf{p}\rangle, \langle \boldsymbol\mu^{Y}, \textbf{p}\rangle \right) = f\left(\int_{\beta \in \mathcal{B}} \mu^{X}(\beta) p(\beta) d\beta, \int_{\beta \in \mathcal{B}} \mu^{Y}(\beta) p(\beta) d\beta \right)
    \\ \text{s.t.} \quad & \textbf{p} \in {S} := \left\{{p}(\beta) \geq 0, \forall \beta \in \mathcal{B} ~|~ {p}(\beta) = \delta_{\beta_0}(\beta), \forall \beta_0 \in \mathcal{B} \right\}
\end{aligned}
\end{equation*}
where $\delta_{\beta_0}(\cdot)$ is the standard delta function defining at $\beta_0 \in \mathcal{B}$ satisfying 
$$
\delta_{\beta_0}(\beta) = \begin{cases}
  \infty, \beta = \beta_0 \\
  0, \text{otherwise}
\end{cases} 
\quad \text{and} \quad
\int_{\beta\in\mathcal{B}} \delta_{\beta_0}(\beta)d\beta =1
$$

(for the probabilistic parameter in continuous  space)
\begin{equation*}
\begin{aligned}
    \max_{{ p}} \quad & f\left(\langle \boldsymbol\mu^{X}, \textbf{p}\rangle, \langle \boldsymbol\mu^{Y}, \textbf{p}\rangle \right) = f\left(\int_{\beta \in \mathcal{B}} \mu^{X}(\beta) p(\beta) d\beta, \int_{\beta \in \mathcal{B}} \mu^{Y}(\beta) p(\beta) d\beta \right)
    \\ \text{s.t.} \quad & \textbf{p} \in \bar{S} := \left\{{p}(\beta) \geq 0, \beta \in \mathcal{B}, ~|~ \int_{\beta\in \mathcal{B}}p(\beta)d\beta =1 \right\}.
\end{aligned}
\end{equation*}

In sum, we would argue that the continuous space do share the same formulation as the discrete space. The gain of probabilistic parameters still holds in the continuous parameter space case.

\section{Proofs for Theorem~\ref{main:theorem_convex_convergence}  and Theorem~\ref{main:theorem_convex} }

We first provide proof for Theorem~\ref{main:theorem_convex} on the regret bound of our proposed \ouralg ~ in Algorithm ~\ref{alg:TS-SCGD} since it requires a more general condition. We then turn to the convergence bound in Theorem~\ref{main:theorem_convex_convergence}. 

We begin by establishing a few technical lemmas. 
\begin{lemma}\label{lemma:grad_bias}
In \ouralg~ where $\bm{g}_t = \hat{V}_t^\top \nabla f(\hat{V}_t \bp_t)$ (as in Eq~\eqref{eq:gt}), we have,
\begin{align*}
\E\left\|\bg_t - \frac{\bm{d}}{\bm{dp}}f\left(\frac{1}{t}\sum_{s} \bmu_s \bp_t\right)\right\|^2\leq \frac{2\left(G_f^2 + G_U^2L_f^2\right)G_U^2\left(\sum_{s=1}^t\frac{1}{\epsilon_s}\right)}{ t^2}.
\end{align*}
\end{lemma}
\begin{proof}
Notice that we have $\E_t \hat{U}_t = \bmu_t$, which means
\begin{align*}
&\E\left\|\frac{1}{t}\sum_{s=1}^t\left(\bmu_s - \hat{U}_s\right)  \right\|^2\\ 
= & \frac{1}{t^2}\sum_{s=1}^t\E\left\|\bmu_s - \hat{U}_s\right\|^2.
\end{align*}
In order to upper bound $\E\|\bmu_s - \hmu_s\|^2$, we recall the technical conditions in Assumption \ref{ass}
\begin{align*}
\E\left\|\bmu_s - \hat{U}_s\right\|^2 \leq \sum_{i=1}^K p_i(s)\left\|\frac{\hat{U}_s^{(i)} - \bmu_{s}^{(i)} }{p_i(s)}  \right\|^2 \leq \sum_{i=1}^K \frac{\left\|\hat{U}_s^{(i)}\right\|^2 + \left\|\bmu_{s}^{(i)}\right\|^2 }{p_i(s)}\leq 2\frac{G_U^2}{\epsilon_s}, 
\end{align*}
where the last inequality is true since only one of the $K$ terms in the summation is non-zero.  Now, recall that $\widehat{V}_t:=\frac{1}{t}\sum_{s=1}^t\hmu_s$ and analogously let  $V_t := \frac{1}{t}\sum_{s=1}^t\bmu_s$, the inequality above would  leads to
\begin{align*}
\E\left\|\widehat{V}_t - V_t  \right\|^2\leq \frac{G_U^2\left(\sum_{s=1}^t\frac{1}{\epsilon_s}\right)}{t^2}.
\end{align*}
With this , we can decompose and bound the  difference of $\bg_t$ and $\nabla f_t\left(\bp_t \right)$ as
\begin{align*}
\E\left\|\bg_t -  \frac{\bm{d}}{\bm{dp}}f\left(\frac{1}{t}\sum_{s} \bmu_s \bp_t\right)\right\|^2 =& \E\left\| \widehat{V}_t\nabla f\left(\widehat{V}_t\bp_t\right) - V_t\nabla f(V_t\bp_t)   \right\|^2\\
= & \E\left\| \left(\widehat{V}_t - V_t\right)\nabla f\left(\widehat{V}_t\bp_t\right) - V_t\left(\nabla f(V_t\bp_t) - \nabla f\left(\widehat{V}_t\bp_t\right)\right)  \right\|^2\\
\leq & 2\E\left\| \left(\widehat{V}_t - V_t\right)\nabla f\left(\widehat{V}_t\bp_t\right) \right\|^2 + 2\E\left\|  V_t\left(\nabla f(V_t\bp_t) - \nabla f\left(\widehat{V}_t\bp_t\right)\right)  \right\|^2\\
\leq & 2G_f^2\E\left\| \widehat{V}_t - V_t\right\|^2 + 2G_U^2\E\left\|  \nabla f(V_t\bp_t) - \nabla f\left(\widehat{V}_t\bp_t\right)  \right\|^2\\
\leq &2G_f^2\E\left\| \widehat{V}_t - V_t\right\|^2 + 2G_U^2L_f^2\E\left\|V_t\bp_t - \widehat{V}_t \bp_t \right\|^2\\
\leq &2G_f^2\E\left\| \widehat{V}_t - V_t\right\|^2 + 2G_U^2L_f^2\E\left\|V_t - \widehat{V}_t \right\|^2\\
\leq & \frac{2(G_f^2 + G_U^2L_f^2)G_U^2\left(\sum_{s=1}^t\frac{1}{\epsilon_s}\right)}{ t^2}.
\end{align*}
\end{proof}

\begin{lemma}\label{lemma:grad_bias_iid}
In \ouralg~ algorithm, if $\bmu_t$ are sampled independently from the same distribution, i.e. $\E\bmu_t = \bmu,\,\forall t$,  then we have
\begin{align*}
\E\left\|\bg_t - \nabla f(\bp_t)\right\|^2\leq \frac{2\left(G_f^2 + G_U^2L_f^2\right)G_U^2\left(\sum_{s=1}^t\frac{1}{\epsilon_s}\right)}{ t^2}.
\end{align*}
\end{lemma}
\begin{proof} (The result can also be obtained as a special case of Lemma~\ref{lemma:grad_bias}.  Notice that we have $\E_t \hmu_t = \bmu$, which means
\begin{align*}
&\E\left\|\frac{1}{t}\sum_{s=1}^t(\bmu - \hmu_s)  \right\|^2\\
= & \frac{1}{t^2}\sum_{s=1}^t\E\|\bmu - \hmu_s\|^2.
\end{align*}
In order to upper bound $\E\|\bmu - \hmu_s\|^2$, we have
\begin{align*}
\E\|\bmu - \hmu_s\|^2 \leq \sum_{i=1}^K p_i(s)\left\|\frac{\hmu_s^{(i)}}{p_i(s)}  \right\|^2 = \sum_{i=1}^K \frac{\left\|\hmu_s^{(i)}\right\|^2}{p_i(s)}\leq \frac{G_U^2}{\epsilon_s}, 
\end{align*}
Denote  $\widehat{V}_t:=\frac{1}{t}\sum_{s=1}^t\hmu_s$, the inequality above would  leads to
\begin{align*}
\E\left\|\widehat{V}_t - \bmu  \right\|^2\leq \frac{G_U^2\left(\sum_{s=1}^t\frac{1}{\epsilon_s}\right)}{t^2}.
\end{align*}
The difference between $\bg_t$ and $\nabla f_t\left(\bp_t \right)$ can be decomposed by
\begin{align*}
\E\|\bg_t - \nabla f_t(\bp_t)\|^2 =& \E\left\| \widehat{V}_t\nabla f\left(\widehat{V}_t\bp_t\right) - \bmu\nabla f(\bmu\bp_t)   \right\|^2\\
= & \E\left\| \left(\widehat{V}_t - \bmu\right)\nabla f\left(\widehat{V}_t\bp_t\right) - \bmu\left(\nabla f(\bmu\bp_t) - \nabla f\left(\widehat{V}_t\bp_t\right)\right)  \right\|^2\\
\leq & 2\E\left\| (\widehat{V}_t - \bmu)\nabla f(\widehat{V}_t\bp_t) \right\|^2 + 2\E\left\|  \bmu\left(\nabla f(\bmu\bp_t) - \nabla f\left(\widehat{V}_t\bp_t\right)\right)  \right\|^2\\
\leq & 2G_f^2\E\left\| \widehat{V}_t - \bmu\right\|^2 + 2G_U^2\E\left\|  \nabla f(\bmu\bp_t) - \nabla f\left(\widehat{V}_t\bp_t\right)  \right\|^2\\
\leq &2G_f^2\E\left\| \widehat{V}_t - \bmu\right\|^2 + 2G_U^2L_f^2\E\left\|\bmu\bp_t - \widehat{V}_t \bp_t \right\|^2\\
\leq &2G_f^2\E\left\| \widehat{V}_t - \bmu\right\|^2 + 2G_U^2L_f^2\E\left\|\bmu - \widehat{V}_t \right\|^2\\
\leq & \frac{2\left(G_f^2 + G_U^2L_f^2\right)G_U^2\left(\sum_{s=1}^t\frac{1}{\epsilon_s}\right)}{ t^2},
\end{align*}
completing the proof.
\end{proof}

\begin{lemma}\label{lemma:key}
For \ouralg, for any probability distribution $\bp^* \in \bar{S}$ in the $K$-dimensional probabilistic simplex, we have for any $t$
\begin{align*}
\sum_{t=1}^T\E\langle \bg_t,\bp^* - \bp_t\rangle \leq  \frac{\ln K}{\gamma} + 4\gamma(e - 2)\sum_{t=1}^T\E\|\bg_t\|^2 + 2\sum_{t=1}^T\epsilon_t\E\langle \bg_t,\bp_t\rangle.
\end{align*}
\end{lemma}
\begin{proof}
Denote $W(t) := \sum_{k=1}^K  \omega_k(t)$, we have 
\begin{align*}
\frac{W(t+1)}{W(t)} =& \sum_{k=1}^K\frac{\omega_k(t+1)}{W(t)}\\
= & \sum_{k=1}^K\frac{\omega_k(t)}{W(t)}\exp{\left(\gamma\g_k(t)\right)}\\
= & \sum_{k=1}^K\frac{p_k(t) - \epsilon_t/K}{1-\epsilon_t}\exp{\left(\gamma\g_k(t)\right)}.
\end{align*}
If $\gamma\g_k(t)\leq 1$, then we get
\begin{align*}
\frac{W(t+1)}{W(t)} \leq& \sum_{k=1}^K\frac{p_k(t) -\epsilon_t/K}{1-\epsilon_t}\left(1+ \gamma\g_k(t) + (e-2)\gamma^2\left(\hg_k(t)\right)^2\right)\\
= & 1 + \sum_{k=1}^K\frac{p_k(t) -\epsilon_t/K}{1-\epsilon_t}\left( \gamma\g_k(t) + (e-2)\gamma^2\left(\hg_k(t)\right)^2\right)\\
\leq & 1 + \frac{\gamma}{1-\epsilon_t}\sum_{k=1}^K p_k(t)\hg_k(t) + \frac{\gamma^2(e-2)}{1-\epsilon_t}\sum_{k=1}^K p_k(t)\left(\hg_k(t)\right)^2\\
\leq &1 + \frac{\gamma}{1-\epsilon_t}\sum_{k=1}^K p_k(t)\hg_k(t) + \frac{\gamma^2(e-2)}{1-\epsilon_t}\sum_{k=1}^K \left(\hg_k(t)\right)^2\\
\leq &1 + \gamma(1+4\epsilon_t)\sum_{k=1}^K p_k(t)\hg_k(t) + \gamma^2(e-2)(1+4\epsilon_t)\sum_{k=1}^K \left(\hg_k(t)\right)^2\quad\left(\text{due to~}\frac{1}{1-\epsilon_t}\leq 1+ 4\epsilon_t\right)\\
= &1 + \gamma(1+4\epsilon_t)\langle \bp_t,\bg_t\rangle + \gamma^2(e-2)(1+4\epsilon_t)\sum_{k=1}^K \left\|\hg_k(t)\right\|^2.
\end{align*}
Therefore we have
\begin{align*}
\ln\left(\frac{W(T+1)}{W(1)}\right)=& \sum_{t=1}^T\ln\left(\frac{W(t+1)}{W(t)}\right)\\
\leq &  \sum_{t=1}^T\ln\left(1 + \gamma(1+4\epsilon_t)\langle \bp_t,\bg_t\rangle + \gamma^2(e-2)(1+4\epsilon_t)\sum_{k=1}^K \left\|\hg_k(t)\right\|^2\right)\\
\leq & \sum_{t=1}^T\left(\gamma(1+4\epsilon_t)\langle \bp_t,\bg_t\rangle + \gamma^2(e-2)(1+4\epsilon_t)\sum_{k=1}^K \left\|\hg_k(t)\right\|^2\right),
\end{align*}
which gives us
\begin{align*}
\E\ln(W(T+1)) - \E\ln(W(1))\leq \gamma\sum_{t=1}^T(1+4\epsilon_t)\E\langle \bg_t,\bp_t\rangle + \gamma^2(e-2)\sum_{t=1}^T(1+4\epsilon_t)\E\|\bg_t\|^2.\numberthis\label{proof:regret_eq1}
\end{align*}
For $W(T+1)$, with any probability distribution $\bp^*= (p^*_1,\cdots,p^*_K)$, we have
\begin{align*}
\ln(W(T+1)) =& \ln\left(\sum_{k}\omega_k(T+1)  \right)\\
= & \sum_{j=1}^Kp^*_j\ln\left(\sum_{k}\omega_k(T+1)  \right)\\
\geq & \sum_{j=1}^Kp^*_j\ln\left(\omega_j(t)  \right)\\
= & \sum_{j=1}^Kp^*_j\ln\left(\exp\left(\gamma\sum_{t=1}^T \hg_j(t) \right)  \right)\\\
= & \sum_{j=1}^K\sum_{t=1}^T \gamma p^*_j \g_j(t),
\end{align*}
then after taking expectation,  we have
\begin{align*}
\E\ln(W(T+1)) \geq \E\left[\sum_{i=1}^K\sum_{t=1}^T \gamma p^*_i \E_t\g_i(t)\right]
=  \gamma\sum_{t=1}^T\E\langle \bm{p}^*,\bg_t\rangle.\numberthis\label{proof:regret_eq2}
\end{align*}
Combing \eqref{proof:regret_eq1} and \eqref{proof:regret_eq2}, we get
\begin{align*}
\gamma\sum_{t=1}^T(1+4\epsilon_t)\E\langle \bg_t,\bp_t\rangle \geq \gamma\sum_{t=1}^T\E \langle \bm{p}^*,\bg_t\rangle - \ln K - \gamma^2(e-2)\sum_{t=1}^T(1+4\epsilon_t)\E\|\bg_t\|^2.
\end{align*}
After rearrangement, the inequality above leads to
\begin{align*}
\sum_{t=1}^T\E\langle \bg_t,\bp_t\rangle \geq \sum_{t=1}^T\E\langle \bm{p}^*,\bg_t\rangle - \frac{\ln K}{\gamma} - 4\gamma(e - 2)\sum_{t=1}^T\E\|\bg_t\|^2 - 2\sum_{t=1}^T\epsilon_t\E\langle \bg_t,\bp_t\rangle.
\end{align*}
Therefore, we get
\begin{align*}
\sum_{t=1}^T\E\langle \bg_t,\bp^* - \bp_t\rangle \leq  \frac{\ln K}{\gamma} + 4\gamma(e - 2)\sum_{t=1}^T\E\|\bg_t\|^2 + 2\sum_{t=1}^T\epsilon_t\E\langle \bg_t,\bp_t\rangle.
\end{align*}
\end{proof}

\section{Proof to Theorem~\ref{main:theorem_convex}}
\begin{proof}
In this proof, we use $f_t(\bp)$ as a short notation for $f(\frac{1}{t} \sum_{s=1}^{t}\bmu_{s} \bp)$. 
Notice that $f_t(\bp_t) - f_t(\bp^*)$ can be bounded by $\langle \nabla f_t(\bp_t),\bp^* - \bp_t\rangle$ by the convexity assumption. So we seek to upper bound $\sum_{t=1}^T\E\langle \nabla f_t(\bp_t),\bp^* - \bp_t\rangle$ to proof Theorem~\ref{main:theorem_convex}. With Lemma~\ref{lemma:key}, we get
\begin{align*}
&\sum_{t=1}^T\E\langle \nabla f_t(\bp_t),\bp^* - \bp_t\rangle\\
\leq& \frac{\ln K}{\gamma} + 4\gamma(e - 2)\sum_{t=1}^T\E\|\bg_t\|^2 + 2\sum_{t=1}^T\epsilon_t\E\langle \bg_t,\bp_t\rangle + \sum_{t=1}^T\E\langle \nabla f_t(\bp_t)-\bg_t,\bp^* - \bp_t\rangle\\
\leq & \frac{\ln K}{\gamma} + 4\gamma(e - 2)\sum_{t=1}^T\E\|\bg_t\|^2 + 2\sum_{t=1}^T\epsilon_t\E\langle \bg_t,\bp_t\rangle + \sum_{t=1}^T4\E\| \nabla f_t(\bp_t)-\bg_t\|^2.
\end{align*}
The last inequality is true since $\langle a, b \rangle \leq \| a \|^{2} \|b \|^{2} $ and $\|\bp^{*} - \bp \| \leq 2$ for any probabilistic vectors $\bp \in \bar{S}$. 
Now by using Lemma~\ref{lemma:grad_bias}, we get
\begin{align*}
&\sum_{t=1}^T\E\langle \nabla f_t(\bp_t),\bp^* - \bp_t\rangle\\
\leq&\frac{\ln K}{\gamma} + 4\gamma(e - 2)\sum_{t=1}^T\E\|\bg_t\|^2 + 2\sum_{t=1}^T\epsilon_t\E\langle \bg_t,\bp_t\rangle + \sum_{t=1}^T\frac{8\left(G_f^2 + G_U^2L_f^2\right)G_U^2\left(\sum_{s=1}^t\frac{1}{\epsilon_s}\right)}{t^2}\\
\leq&\frac{\ln K}{\gamma} + 4\gamma(e - 2)G^2T + 2G\sum_{t=1}^T\epsilon_t + \sum_{t=1}^T\frac{8\left(G_f^2 + G_U^2L_f^2\right)G_U^2\left(\sum_{s=1}^t\frac{1}{\epsilon_s}\right)}{t^2}.
\end{align*}
Setting $\gamma=\frac{1}{\sqrt{T}}$ and $\epsilon_t = \frac{G}{\sqrt{t+1}}$, we can ensure that $\gamma\g_i(t)=\frac{\gamma \g_i(t)}{p_i(t)}\leq \frac{\gamma G}{\epsilon_t}\leq 1$, then the inequality above becomes
\begin{align*}
&\sum_{t=1}^T\E\left(f_t(\bp^*) - f_t(\bp_t)\right)\\
\leq&\sum_{t=1}^T\E\langle \nabla f_t(\bp_t),\bp^* - \bp_t\rangle\\
\leq & \sqrt{T}\left(\ln K + 4(e - 1)G^2 + 4G^2 +\frac{32\left(G_f^2 + G_U^2L_f^2\right)G_U^2}{G}\right).
\end{align*}

For the second part of Theorem~\ref{main:theorem_convex}, we first decompose the loss as 

\begin{align*}
&\sum_{t=1}^T f\left(\frac{1}{t}\sum_{s=1}^t \bm\mu_s\bp^{*} \right) - \sum_{t=1}^T\E\left[f\left(\frac{1}{t}\sum_{s=1}^t \hat{U}_s \bp_t  \right)\right] \\
\leq & \sum_{t} f\left(\frac{1}{t}\sum_{s} \bm\mu_s\bp^{*} \right) -\sum_t\E f\left(\frac{1}{t} \sum_{s}\bmu_{s} \bp_t \right) + \sum_t\E f\left(\frac{1}{t} \sum_{s}\bmu_{s} \bp_t \right) - \sum_{t=1}^T\E\left[f\left(\frac{1}{t}\sum_{s=1}^t \hat{U}_s \bp_t  \right)\right] \\
\leq &  {\sqrt{T}}C_1 + \sum_t\E \left[ f\left(\frac{1}{t} \sum_{s}\bmu_{s} \bp_t \right) - f\left(\frac{1}{t}\sum_{s} \hat{U}_s \bp_t  \right)\right],
\end{align*} 
where the first term is upper bounded based on Theorem~\ref{main:theorem_convex} and $C_1$ is the constant in the Theorem. To bound the second term, note that 
\begin{align*}
 \E  f(V_t \bp_t) - f(\hat{V}_t \bp_t) \leq & \E G\| \bp_t \|^{2} \| V_t - \hat{V}_t \|^{2}
 \leq  G G_u^{2} \left(\sum_{s} \frac{1}{\epsilon_s}\right)/ t^{2}
\end{align*}
and recall that $\epsilon_t = \frac{G}{\sqrt{t+1}}$,
we have 
\begin{align*}
    &\sum_{t=1}^T f\left(\frac{1}{t}\sum_{s=1}^t \bm\mu_s\bp^{*} \right) - \sum_{t=1}^T\E\left[f\left(\frac{1}{t}\sum_{s=1}^t \hat{U}_s \bp_t  \right)\right] \\
\leq  &  {\sqrt{T}}C_1  + \sum_{t=1}^{T} G_{U}^{2}\sqrt{t}/t^2 \\
\leq &  {\sqrt{T}} \left(C_1  + 2G_{U}^{2}\right).
\end{align*}
It completes the proof.
\end{proof}

\section{Proof to Theorem~\ref{main:theorem_convex_convergence}}
\begin{proof}
Notice that $f(\bp_t) - f_t(\bp^*)$ can be bounded by $\langle \nabla f(\bp_t),\bp^* - \bp_t\rangle$, so in order to upper bound $\sum_{t=1}^T\E\langle \nabla f(\bp_t),\bp^* - \bp_t\rangle$, by using Lemma~\ref{lemma:key}, we get
\begin{align*}
&\sum_{t=1}^T\E\langle \nabla f(\bp_t),\bp^* - \bp_t\rangle\\
\leq& \frac{\ln K}{\gamma} + 4\gamma(e - 2)\sum_{t=1}^T\E\|\bg_t\|^2 + 2\sum_{t=1}^T\epsilon_t\E\langle \bg_t,\bp_t\rangle + \sum_{t=1}^T\E\langle \nabla f(\bp_t)-\bg_t,\bp^* - \bp_t\rangle\\
\leq & \frac{\ln K}{\gamma} + 4\gamma(e - 2)\sum_{t=1}^T\E\|\bg_t\|^2 + 2\sum_{t=1}^T\epsilon_t\E\langle \bg_t,\bp_t\rangle + \sum_{t=1}^T4\E\| \nabla f(\bp_t)-\bg_t\|^2.
\end{align*}
Now by using Lemma~\ref{lemma:grad_bias_iid}, we get
\begin{align*}
&\sum_{t=1}^T\E\langle \nabla f(\bp_t),\bp^* - \bp_t\rangle\\
\leq&\frac{\ln K}{\gamma} + 4\gamma(e - 2)\sum_{t=1}^T\E\|\bg_t\|^2 + 2\sum_{t=1}^T\epsilon_t\E\langle \bg_t,\bp_t\rangle + \sum_{t=1}^T\frac{8\left(G_f^2 + G_U^2L_f^2\right)G_U^2\left(\sum_{s=1}^t\frac{1}{\epsilon_s}\right)}{t^2}\\
\leq&\frac{\ln K}{\gamma} + 4\gamma(e - 2)G^2T + 2G\sum_{t=1}^T\epsilon_t + \sum_{t=1}^T\frac{8\left(G_f^2 + G_U^2L_f^2\right)G_U^2\left(\sum_{s=1}^t\frac{1}{\epsilon_s}\right)}{t^2}.
\end{align*}
Setting $\gamma=\frac{1}{\sqrt{T}}$ and $\epsilon_t = \frac{G}{\sqrt{t+1}}$, we can ensure that $\gamma\g_i(t)=\frac{\gamma \g_i(t)}{p_i(t)}\leq \frac{\gamma G}{\epsilon_t}\leq 1$, then the inequality above leads to
\begin{align*}
&f(\bp^*) - \E f\left(\frac{1}{T}\sum_{t=1}^T\bp_t\right)\\
\leq&\frac{1}{T}\sum_{t=1}^T\E\langle \nabla f(\bp_t),\bp^* - \bp_t\rangle\\
\leq & \frac{1}{\sqrt{T}}\left(\ln K + 4(e - 1)G^2 + 4G^2 +\frac{32\left(G_f^2 + G_U^2L_f^2\right)G_U^2}{G}\right).
\end{align*}
It completes the proof.
\end{proof}

\end{document}